\documentclass[conference]{IEEEtran}

\makeatletter
\def\ps@IEEEtitlepagestyle{%
  \def\@oddfoot{\mycopyrightnotice}%
  \def\@evenfoot{}%
}
\def\mycopyrightnotice{%
  {\footnotesize\hfill}%
  \gdef\mycopyrightnotice{}
}
\makeatother

\usepackage{cite}
\usepackage{amsmath,amssymb,amsfonts}
\usepackage{amsthm}
\usepackage{graphicx}
\usepackage{booktabs}
\usepackage{textcomp}
\usepackage{xcolor}
\usepackage{microtype}
\usepackage{eso-pic}

\newtheorem{proposition}{Proposition}
\theoremstyle{definition}
\newtheorem{definition}{Definition}
\theoremstyle{remark}
\newtheorem{remark}{Remark}

\begin{document}

\title{\vspace*{1cm}The Price of Correlated Tests:\\
How Strict Should a Model Release Gate Be?}

\author{\IEEEauthorblockN{Marco Pollanen}
\IEEEauthorblockA{\textit{Department of Mathematics \& Statistics} \\
\textit{Trent University}\\
Peterborough, ON, Canada \\
marcopollanen@trentu.ca}
}


\maketitle

\begin{abstract}
Before a machine learning model ships, it often has to pass a suite of
automated tests. Requiring every test to pass looks safe, yet it can reject
many models that would have served users well, and it does not say how
trustworthy a passing model actually is. We treat the release gate as a design
problem: choose how many tests a model must pass so that cleared models meet a
stated reliability target, while keeping as many good models as possible. A
two-class latent-factor model makes both costs explicit and reduces each
calculation to a one-dimensional integral. We prove that when both classes
share the same latent correlation, a stricter gate always raises reliability,
so the gate that keeps the most good models is the most lenient one that still
meets the target. Under pass-all gating, any reliability target short of
perfection is attainable within the model, but the share of good models kept
tends to zero as the suite grows. Correlation between tests sets the price. In
one configuration, a 99 percent target needs 8 independent tests, but 74 tests at a latent correlation of 0.3 and 5,182 at
0.5, where the gate keeps fewer than one good model in ten. We also give a
validation procedure, built on exact binomial bounds, that certifies a gate
from labelled data even when the gate is chosen from a fixed shortlist.
\end{abstract}


\section{Introduction}
\label{sec:intro}

A training sweep produces hundreds of checkpoints. A release review must decide
which of them may be deployed, and it often does so with a fixed suite of
automated checks \cite{breck2017,sculley2015}. The gate can fail in two ways.
It can clear a model that violates the deployment requirement, or it can reject
a model that meets it. The common convention, ``every test must pass,'' guards
against the first failure and says nothing about the second.

How many of $k$ tests should a model be required to pass? The answer turns on
three quantities. The first is how often each test errs. The second is whether
the tests err \emph{together}: a flaw that slips past one benchmark tends to slip
past similar ones, and a good model that trips one check tends to trip related
checks. The third is the composition of the pool. When good candidates are rare,
even a small false-clearance rate can fill the released set with bad ones. With
1,000 flawed candidates for every 10 good ones, a 99\% reliability target
requires the gate to clear flawed candidates at no more than one ten-thousandth
of the rate at which it clears good ones.

Diagnostic medicine has long studied how to combine imperfect tests and how
conditional dependence distorts their apparent accuracy
\cite{vacek1985,marshall1989}; latent-class models with random effects represent
that dependence explicitly \cite{qu1996}. The equicorrelated one-factor model we
use also underlies portfolio credit-risk calculations \cite{vasicek2002}. In
machine learning, selective classification trades risk against coverage when a
model abstains on some inputs \cite{geifman2017}, and learn-then-test calibrates
such settings with finite-sample guarantees by treating their selection as
multiple testing \cite{angelopoulos2025}. Our setting differs in two respects:
the objects screened are candidate models, not inputs, and retention is measured
within the acceptable class rather than as the overall fraction accepted. Work on
leaderboard reuse addresses yet another problem, adaptive overfitting to a shared
test set \cite{blum2015,roelofs2019}. We contribute an analysis of correlated
benchmark pass counts for candidate-model release, in four parts:
\begin{itemize}
\item \emph{A design rule with a proof.} For an $r$-of-$k$ gate we compute both
the false-clearance rate and the rate at which good models survive. In a
Gaussian one-factor model with a common within-class correlation, precision
rises strictly with the required pass count for a fixed suite, and with suite
size under pass-all gating. Under pass-all gating, any precision target below
one is attainable, but sensitivity tends to zero as the suite grows. Among
$r$-of-$k$ gates meeting the target, the most lenient keeps the most good models.
\item \emph{The price of dependence, quantified.} Holding per-test error rates
fixed, raising latent correlation from $0$ to $0.5$ raises the pass-all suite
needed for 99\% precision from 8 tests to 5,182, and cuts the share of good
models kept from $66\%$ to under $10\%$.
\item \emph{A worked release decision.} In an illustrative sweep, misjudging the test
correlation by a factor of two changes the outcome more than moving the gate by
one test.
\item \emph{A certificate from data.} Exact binomial bounds with a union bound
give a simultaneous lower bound on precision, so a team can select a gate from
a fixed shortlist and still hold a valid guarantee. We also quantify how much
labelled data that requires.
\end{itemize}

Two scope notes. Our precision is the probability
that a candidate is acceptable given that it cleared, for a candidate drawn before
evaluation; it is not the probability that a winner picked afterwards from the
cleared set is acceptable. The tables are model calculations for stated
parameter values, not measurements of deployed pipelines.
Section~\ref{sec:certify} shows how to turn such calculations into guarantees.

\section{Model: From Test Outcomes to Reliable Clearance}
\label{sec:model}

\subsection{Candidates, tests and the gate}

Fix a deployment criterion, such as a quality or safety requirement, and call
each candidate \emph{acceptable} or \emph{unacceptable} with respect to it. This
label belongs to the candidate, not to the tests, which detect it only
imperfectly. The pool holds $n_+$ acceptable and $n_-$ unacceptable candidates,
with imbalance $\Omega=n_-/n_+$. Each candidate is evaluated on $k$ tests and
released if it passes at least $r$ of them. The tests are noisy indicators of
one common criterion. An $r$-of-$k$ rule is therefore not a licence to waive
requirements that are each mandatory in their own right; those belong in
separate pass-all checks.

\begin{definition}[Per-test error rates]
\label{def:rates}
The \emph{miss rate} $p$ is the probability that an unacceptable candidate passes
a given test. The \emph{false-alarm rate} $\eta$ is the probability that an
acceptable candidate fails a given test.
\end{definition}

These are Type~II and Type~I error rates measured against the deployment
criterion, not benchmark accuracies. A benchmark verdict is usually a thresholded
average over many items; that sampling noise is absorbed into $p$ and $\eta$.
Probabilities are taken over candidates within a class and over evaluation
randomness.

\subsection{Dependence between tests}

A shared latent factor captures the tendency of tests to err together. Let
$\phi$ and $\Phi$ denote the standard normal density and distribution function.
Candidate $i$ in class $\mathcal C\in\{+,-\}$ draws $\theta_i\sim\mathcal N(0,1)$
and receives latent scores
\begin{equation}
\label{eq:onefactor}
S_{ij}=\sqrt{\rho}\,\theta_i+\sqrt{1-\rho}\,\varepsilon_{ij},
\qquad\varepsilon_{ij}\sim\mathcal N(0,1)\ \text{i.i.d.},
\end{equation}
with the noise terms independent of $\theta_i$.
It passes test $j$ when $S_{ij}>c_{\mathcal C}$, with thresholds
$c_-=\Phi^{-1}(1-p)$ and $c_+=\Phi^{-1}(\eta)$ matching the per-test rates. For an
unacceptable candidate, $\theta_i$ measures how well its flaw hides from the
suite; for an acceptable one, how clearly the suite recognises its competence.
The correlation $\rho$ is common to both classes; Section~\ref{sec:certify}
relaxes this.

\begin{remark}[What $\rho$ is, and is not]
\label{rem:whatrho}
$\rho$ is the correlation of \emph{latent} scores \emph{within one class}. It is
not the Pearson correlation of benchmark percentages, not the correlation of
binary pass indicators, and not a correlation computed across a mixed population
of good and bad models. Even inside the model, at threshold zero the pass
indicators have correlation $\tfrac{2}{\pi}\arcsin\rho$, not $\rho$.
\end{remark}

\subsection{Clearance probabilities and precision}

\begin{proposition}[Two-sided clearance]
\label{prop:clear}
Let $\pi_{\mathcal C}(\theta)=\Phi\big((\sqrt{\rho}\,\theta-c_{\mathcal C})/\sqrt{1-\rho}\big)$.
Under \eqref{eq:onefactor}, a candidate of class $\mathcal C$ clears an
$r$-of-$k$ gate with probability
\begin{equation}
\label{eq:clear}
G_{\mathcal C}(r,k,\rho)=\int_{-\infty}^{\infty}
\Pr\big(\mathrm{Bin}(k,\pi_{\mathcal C}(\theta))\ge r\big)\,\phi(\theta)\,d\theta .
\end{equation}
\end{proposition}

\begin{proof}
Given $\theta_i=\theta$ the noise terms are independent and each test is passed
with probability $\pi_{\mathcal C}(\theta)$, so the pass count is
$\mathrm{Bin}(k,\pi_{\mathcal C}(\theta))$. Integrate over $\theta$.
\end{proof}

We write $q=G_-$ for the \emph{false-clearance} probability and $\sigma=G_+$ for
the \emph{sensitivity}, the share of good candidates the gate keeps. Both depend
on $r$ through \eqref{eq:clear}, and raising $r$ lowers both. The choice of $r$
therefore trades false clearances against lost good candidates.

\begin{definition}[Deployment precision]
\label{def:ppv}
For a candidate drawn uniformly from the pool before evaluation, the probability
that it is acceptable given that it clears is
\begin{equation}
\label{eq:ppv}
\Pi(r)=\frac{n_+\,\sigma(r)}{n_+\,\sigma(r)+n_-\,q(r)}
=\Big(1+\Omega\,\frac{q(r)}{\sigma(r)}\Big)^{-1}.
\end{equation}
\end{definition}

Equivalently, $\Pi$ is the ratio of expected acceptable clearances to expected
total clearances. It is not the realised acceptable fraction in one particular
cleared pool, which varies from sweep to sweep. Because the gate enters only
through $q/\sigma$, a precision target becomes a simple constraint:
\begin{equation}
\label{eq:rule}
\Pi(r)\ge\tau
\quad\Longleftrightarrow\quad
\frac{q(r)}{\sigma(r)}\le\frac{1-\tau}{\tau\,\Omega}.
\end{equation}
A precision target is a ceiling on the ratio of false clearance to sensitivity,
and the ceiling tightens in proportion to the pool imbalance. With $\tau=0.99$
and $\Omega=100$ the ceiling is $1.01\times10^{-4}$. Per-test error rates that
look reasonable can still miss a ceiling that low.

Two system-level quantities aid interpretation. The expected number of false
releases is $n_-q$, whatever the dependence across candidates. If candidates
clear independently, at least one good candidate survives with probability
$1-(1-\sigma)^{n_+}$. Checkpoints from one sweep share a base model, data and
recipe; if their clearance events are positively associated, this expression is
an upper bound.

\section{Choosing the Gate}
\label{sec:rule}

\begin{proposition}[Monotone precision]
\label{prop:mlr}
Suppose $0<p<1-\eta<1$, $0\le\rho<1$ and $n_+,n_->0$. Then:
(i) for $k\ge2$, $\sigma(r)/q(r)$, and hence $\Pi(r)$, is strictly increasing in
$r\in\{1,\ldots,k\}$;
(ii) for pass-all gates ($r=k$), $\Pi$ is strictly increasing in $k$;
(iii) for pass-all gates, $\Pi\to1$ as $k\to\infty$;
(iv) for pass-all gates, $\sigma\to0$ and $q\to0$ as $k\to\infty$.
\end{proposition}

\begin{proof}
For $\rho=0$ the ratio of the two binomial mass functions at count $j$ is
$[(1-\eta)/p]^{j}[\eta/(1-p)]^{k-j}$, which increases in $j$ because
$(1-\eta)(1-p)>p\eta$; the pass-all ratio $[(1-\eta)/p]^k$ grows without bound.
For $0<\rho<1$ let $U_{\mathcal C}=\pi_{\mathcal C}(\theta)$ have density
$f_{\mathcal C}$ on $(0,1)$. Substituting $z=\Phi^{-1}(u)$ into the normal
density gives
\[
L(u):=\frac{f_+(u)}{f_-(u)}
=\exp\Big\{\frac{(c_--c_+)\sqrt{1-\rho}}{\rho}\,z
             +\frac{c_-^2-c_+^2}{2\rho}\Big\},
\]
which is strictly increasing and tends to $\infty$ as $u\to1$, because
$c_->c_+$ when $p<1-\eta$. Let $a_j$ and $b_j$ be the probabilities of pass
count $j$ in the acceptable and unacceptable classes. Then $a_j/b_j$ is the mean
of $L(U)$ under the density $w_j\propto u^j(1-u)^{k-j}f_-(u)$. Moving from $j$
to $j+1$ tilts $w_j$ by the strictly increasing factor $h(u)=u/(1-u)$, and
\[
\mathbb E_{w_{j+1}}[L]-\mathbb E_{w_j}[L]
=\frac{\operatorname{Cov}_{w_j}(h,L)}{\mathbb E_{w_j}[h]}>0,
\]
since $2\operatorname{Cov}(h(U),L(U))=\mathbb E[(h(U)-h(V))(L(U)-L(V))]>0$ for
independent copies $U,V$ of a continuous law. So $a_j/b_j$ increases strictly.
The tail ratio $\sum_{j\ge r}a_j/\sum_{j\ge r}b_j$ is a $b$-weighted average of
these ratios; dropping its smallest term, $j=r$, strictly increases it, which
proves (i). For (ii), $\sigma/q$ is the mean of $L$ under
$v_k\propto u^kf_-(u)$, and $v_{k+1}$ tilts $v_k$ by $u$; the same covariance
argument applies. For (iii), fix $M$ and choose $u_M$ with $L>M$ on $(u_M,1)$.
For any $u'\in(u_M,1)$ the $v_k$-mass of $(0,u_M]$ is at most
$(u_M/u')^k/\Pr(U_->u')\to0$, so $\liminf_k \sigma/q\ge M$. As $M$ is
arbitrary, $q/\sigma\to0$ and $\Pi\to1$ by \eqref{eq:ppv}. For (iv),
$\sigma=\mathbb E[U_+^k]$ and $q=\mathbb E[U_-^k]$ with $U_{\mathcal C}<1$
almost surely, so both tend to zero by dominated convergence (for $\rho=0$,
$\sigma=(1-\eta)^k$ and $q=p^k$).
\end{proof}

Part (i) is a monotone-likelihood-ratio property in the classical sense
\cite{karlin1956}; the proof specialises it to the one-factor model. It also
has an awkward consequence. If a team maximises precision alone, it
always picks pass-all, however many good candidates that discards. Near-identical
precision values can hide very different retention. We therefore select
\begin{equation}
\label{eq:objective}
r^{*}(\tau)=\min\{r:\Pi(r)\ge\tau\},
\end{equation}
the most lenient gate that meets target $\tau$. Sensitivity can only fall as $r$
rises, so $r^*$ keeps the most good candidates among all deterministic
$r$-of-$k$ gates that meet the target. That
argument needs no monotone precision; under common correlation, part (i) adds
that the feasible gates form a contiguous range $\{r^*,\ldots,k\}$. When no gate
qualifies, the rule reports infeasibility instead of defaulting to the strictest
gate.

Parts (iii) and (iv) describe what happens when a pass-all suite grows. For any
$\tau<1$, a large enough suite meets the precision target, provided more tests
with the same error rates and dependence exist. Its sensitivity, however, tends
to zero. For fixed $n_+$ and whatever the dependence between candidates, the
union bound then gives
$\Pr(\text{some acceptable candidate clears})\le n_+\sigma\to0$. A pass-all
suite can thus become almost perfectly precise and almost useless. Adding tests under pass-all cannot remedy a violated minimum-retention requirement.

Vanishing retention is a property of pass-all, not of large suites. With
independent tests, any pass fraction $t\in(p,1-\eta)$ with $r/k\to t$ drives
false clearance to zero exponentially fast while sensitivity tends to one. The
question is how much of that freedom survives dependence. Each $(r,k)$ evaluation is a single one-dimensional
integral, and the search over $r$ is discrete, so the full design is cheap to
compute.

\section{What a Stricter Gate Buys, and What It Costs}
\label{sec:findings}

All calculations in this section share one configuration: miss rate $p=0.3$,
false-alarm rate $\eta=0.05$, and a pool of 1,000 unacceptable against 10
acceptable candidates ($\Omega=100$), with target $\tau=0.99$. We computed
\eqref{eq:clear} by quadrature on a dense grid in log space and confirmed the
extreme cases with adaptive high-precision quadrature.

\begin{table}[t]
\centering
\footnotesize
\setlength{\tabcolsep}{5pt}
\caption{Most lenient gate $r^{*}$ meeting 99\% precision, with its sensitivity
$\sigma(r^*)$ in parentheses. $p=0.3$, $\eta=0.05$, $\Omega=100$. A dash means
no gate meets the target.}
\label{tab:optr}
\begin{tabular}{ccccc}
\toprule
$\rho$ & $k=10$ & $k=20$ & $k=40$ & $k=80$\\
\midrule
$0.0$ & $10\ (0.599)$ & $15\ (1.000)$ & $24\ (1.000)$ & $41\ (1.000)$\\
$0.1$ & --- & $19\ (0.740)$ & $34\ (0.968)$ & $64\ (0.993)$\\
$0.3$ & --- & --- & --- & $80\ (0.282)$\\
$0.5$ & --- & --- & --- & ---\\
\bottomrule
\end{tabular}
\end{table}

\paragraph*{Finding 1: correlated tests force stricter gates}
Table~\ref{tab:optr} fixes the suite and finds the most lenient qualifying gate.
With independent tests the gate relaxes as the suite grows: 15 of 20, then 24 of
40, then 41 of 80, each keeping essentially every good candidate. The required
pass fraction drifts down toward $p$; it reaches $0.43$ at $k=200$ and $0.32$ at
$k=10^4$. At $\rho=0.1$ lenient gates still appear, but later and at some cost in
sensitivity. At $\rho=0.3$ no suite up to 40 tests qualifies, and at 80 tests only
pass-all does, discarding $72\%$ of the good candidates. At $\rho=0.5$ nothing
in the table qualifies. As correlation grows, the freedom to tolerate partial
failures shrinks and then disappears.

\begin{table}[t]
\centering
\footnotesize
\setlength{\tabcolsep}{3.5pt}
\caption{Gate trade-off at $k=20$ (same configuration, upper half of the gate
range). Under independence precision rounds to 1.000 from $r=17$; those gates
differ in precision only beyond three decimals, but differ markedly in
sensitivity. At $\rho=0.3$ even pass-all misses
the target.}
\label{tab:gate}
\begin{tabular}{rccc@{\hspace{8pt}}ccc}
\toprule
& \multicolumn{3}{c}{$\rho=0$} & \multicolumn{3}{c}{$\rho=0.3$}\\
\cmidrule(lr){2-4}\cmidrule(l){5-7}
$r$ & $q$ & $\sigma$ & $\Pi$ & $q$ & $\sigma$ & $\Pi$\\
\midrule
14 & $2.6{\times}10^{-4}$  & 1.000 & 0.975 & $6.6{\times}10^{-2}$ & 0.983 & 0.129\\
15 & $4.3{\times}10^{-5}$  & 1.000 & 0.996 & $4.5{\times}10^{-2}$ & 0.972 & 0.178\\
16 & $5.6{\times}10^{-6}$  & 0.997 & 0.999 & $2.9{\times}10^{-2}$ & 0.954 & 0.250\\
17 & $5.4{\times}10^{-7}$  & 0.984 & 1.000 & $1.7{\times}10^{-2}$ & 0.923 & 0.356\\
18 & $3.8{\times}10^{-8}$  & 0.925 & 1.000 & $8.6{\times}10^{-3}$ & 0.870 & 0.503\\
19 & $1.7{\times}10^{-9}$  & 0.736 & 1.000 & $3.6{\times}10^{-3}$ & 0.768 & 0.683\\
20 & $3.5{\times}10^{-11}$ & 0.358 & 1.000 & $9.4{\times}10^{-4}$ & 0.555 & 0.856\\
\bottomrule
\end{tabular}
\end{table}

Table~\ref{tab:gate} shows the mechanism at 20 tests. Without correlation, each
extra required pass cuts false clearance by a factor of six or more while
sensitivity barely moves, until the last few gates. With $\rho=0.3$, each extra
pass buys a factor between 1.5 and 4, because a flaw that evades several tests
tends to evade the rest. A single ``effective number of independent tests''
generally cannot reproduce both class-specific clearance probabilities at once:
correlation changes $q$ and $\sigma$ by different amounts, and only their ratio
matters. Nor does dependence always destroy information. Two classes
that emit $\{00,11\}$ and $\{01,10\}$ uniformly have identical per-test rates,
yet whether the two outcomes agree separates them perfectly. Our results describe
the one-factor family, not dependence in general.

\begin{figure}[t]
\centering
\includegraphics[width=\columnwidth]{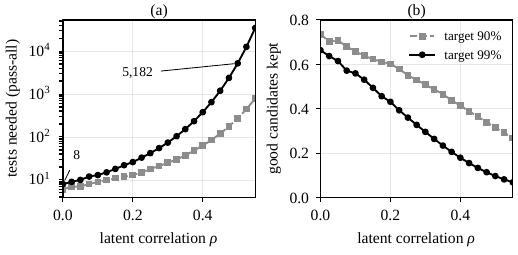}
\caption{The price of dependence under a pass-all gate ($p=0.3$, $\eta=0.05$,
$\Omega=100$). (a) Suite size needed to reach the precision target, on a log
scale. (b) Share of acceptable candidates the gate keeps at that size. Markers
are exact integer minimum-suite calculations; small reversals reflect integer
suite sizes.}
\label{fig:price}
\end{figure}

\paragraph*{Finding 2: more tests restore precision but not retention}
When only pass-all qualifies, the remaining lever is suite size.
Proposition~\ref{prop:mlr}(iii) guarantees that a large enough suite reaches
any target below one; Fig.~\ref{fig:price} and Table~\ref{tab:cost} show the bill.
The suite needed for 99\% precision grows from 8 tests at $\rho=0$ to 74 at
$\rho=0.3$ and 5,182 at $\rho=0.5$, and the growth accelerates as $\rho$
rises. Across the rows of Table~\ref{tab:cost}, sensitivity falls steadily. The 5,182-test suite does reach 99\%
precision, but it releases fewer than one good candidate in ten. Adding tests
buys back precision; it does not buy back the good candidates that pass-all
throws away.

\begin{table}[t]
\centering
\footnotesize
\setlength{\tabcolsep}{6pt}
\caption{Pass-all suite size $k$ needed to reach a precision target, and the
sensitivity $\sigma$ at that size. Same configuration as
Table~\ref{tab:optr}; $\rho$ is the within-class latent correlation of
Remark~\ref{rem:whatrho}.}
\label{tab:cost}
\begin{tabular}{crcrc}
\toprule
& \multicolumn{2}{c}{$\Pi\ge0.90$} & \multicolumn{2}{c}{$\Pi\ge0.99$}\\
\cmidrule(lr){2-3}\cmidrule(l){4-5}
$\rho$ & $k$ & $\sigma$ & $k$ & $\sigma$\\
\midrule
$0.0$ & $6$   & $0.735$ & $8$       & $0.663$\\
$0.1$ & $9$   & $0.657$ & $13$      & $0.559$\\
$0.2$ & $13$  & $0.602$ & $26$      & $0.431$\\
$0.3$ & $25$  & $0.508$ & $74$      & $0.296$\\
$0.4$ & $63$  & $0.413$ & $377$     & $0.179$\\
$0.5$ & $272$ & $0.315$ & $5{,}182$ & $0.097$\\
\bottomrule
\end{tabular}
\end{table}

The two tables answer complementary questions. Table~\ref{tab:optr} fixes the
suite and asks for the most lenient qualifying gate. Table~\ref{tab:cost} fixes
the gate at pass-all and asks how large the suite must be. Where the first
reports infeasibility, the second prices its removal.

\section{A Release Decision Under Uncertain Dependence}
\label{sec:case}

Consider a team whose training sweep has produced 500 checkpoints. Experience
suggests that about 2\% genuinely meet the deployment criterion, so $n_+=10$ and
$n_-=490$ ($\Omega=49$). The suite has $k=12$ benchmarks. The team estimates
that a benchmark misses a flawed checkpoint 30\% of the time ($p=0.3$) and
wrongly fails a good one 5\% of the time ($\eta=0.05$). It judges the benchmarks
moderately correlated and assumes $\rho=0.2$. Choosing one winner among the
cleared checkpoints is a separate step, outside this calculation.

\begin{table}[t]
\centering
\footnotesize
\setlength{\tabcolsep}{4.5pt}
\caption{Worked case: 500 checkpoints, $k=12$, $p=0.3$, $\eta=0.05$,
$n_+=10$, $n_-=490$. $F=n_-q$ is the expected number of false releases. Left:
the assumed $\rho=0.2$. Right: the same gates if the true $\rho$ is $0.4$.}
\label{tab:case}
\begin{tabular}{rccc@{\hspace{8pt}}ccc}
\toprule
& \multicolumn{3}{c}{assumed $\rho=0.2$} & \multicolumn{3}{c}{actual $\rho=0.4$}\\
\cmidrule(lr){2-4}\cmidrule(l){5-7}
$r$ & $\sigma$ & $\Pi$ & $F$ & $\sigma$ & $\Pi$ & $F$\\
\midrule
9  & 0.981 & 0.361 & 17.4 & 0.961 & 0.182 & 43.2\\
10 & 0.948 & 0.569 & 7.2  & 0.928 & 0.265 & 25.7\\
11 & 0.860 & 0.792 & 2.3  & 0.861 & 0.402 & 12.8\\
12 & 0.621 & 0.937 & 0.4  & 0.694 & 0.616 & 4.3\\
\bottomrule
\end{tabular}
\end{table}

Under the team's assumption (Table~\ref{tab:case}, left), a 99\% target is out of
reach at 12 tests: the best available precision is 0.937, at pass-all. By rule
\eqref{eq:objective} the design is infeasible: the target rules out every
available gate. Pass-all maximises precision among them, with 0.4 expected false
releases and, for independent candidates, a 0.9999 chance that at least one good
checkpoint survives, but it still fails the stipulated requirement. Relaxing to
$r=11$ keeps more good checkpoints ($\sigma=0.860$ versus 0.621) but drops
precision to 0.792 and raises expected false releases to 2.3. Shipping under
pass-all therefore means accepting a lower standard, and the team should make
that choice explicitly rather than treat the convention as meeting the target.

The right half of the table shows why the calculation is worth running. If the
true correlation is 0.4 rather than 0.2, the same pass-all gate delivers
precision 0.616, and expected false releases rise from 0.4 to 4.3, nearly
tenfold, from a factor-of-two error in a parameter the team never measured.
That error moves precision by 0.32; moving the gate from 12 to 11 at the
assumed correlation moves it by 0.15. Suite sizing is more fragile still: a
pass-all suite reaching 99\% precision needs 22 tests if $\rho=0.2$ but 216 if
$\rho=0.4$. Before choosing between a stricter gate, a larger suite or a
different set of tests, the team should measure dependence.

\section{From Calculation to Certificate}
\label{sec:certify}

\subsection{What data can identify}

The model has four inputs: $p$, $\eta$, $\rho$ and $\Omega$. Unlabelled
leaderboard score matrices identify none of them without further assumptions.
The error rates need labels saying which candidates meet the criterion, and
$\rho$ must be estimated within those labelled classes. Across a mixed population
of models, variation in overall quality can inflate observed benchmark correlation,
so a factor fitted to such a matrix need not measure shared evasiveness within
the unacceptable class. Public scores still help with model checking: raw and
residual correlations show how much shared variation exists, and latent-trait
models fitted by marginal maximum likelihood \cite{bock1981} can be checked for
structure that one factor misses. If an
independent held-out criterion is available, one can split candidates by it and
estimate $p$, $\eta$ and $\rho$ within each group; transferring those estimates
to true acceptability requires separate evidence that the proxy tracks it.

\subsection{Robust design}

A grid of plausible parameters is a sensitivity analysis. It becomes a guarantee
only when the set it covers has justified coverage. Let
$\vartheta=(p,\eta,\rho_-,\rho_+,\Omega)$, now allowing different correlations in
the two classes, and let $\mathcal U$ be a data-dependent uncertainty set with
$\Pr(\vartheta_0\in\mathcal U)\ge1-\delta$ for the true parameter $\vartheta_0$.
Define
\[
\underline\Pi_{\mathcal U}(r)=\inf_{\vartheta\in\mathcal U}\Pi_\vartheta(r),
\qquad r^*_{\mathcal U}=\min\{r:\underline\Pi_{\mathcal U}(r)\ge\tau\}.
\]
With probability at least $1-\delta$, every gate with
$\underline\Pi_{\mathcal U}(r)\ge\tau$ meets the target within the model family.
Sensitivity falls with $r$ under every parameter value, so $r^*_{\mathcal U}$
keeps the most good candidates among robust-feasible gates. With $\rho_-\ne\rho_+$
precision need not be monotone in $r$, so all gates should be searched. An empty
feasible set means that no gate is certified under this uncertainty set; it
does not show that the target is impossible.

\subsection{Direct certification without a Gaussian model}

The model can be bypassed entirely. Draw $N_-$ unacceptable and $N_+$
acceptable labelled candidates, i.i.d.\ within each class and representative of
the candidates the gate will face. Checkpoints from one training run are
correlated and do not automatically count as independent observations. The
suite, the benchmark thresholds and a shortlist of $J$ gates must be fixed
independently of this sample. For each gate compute a one-sided Clopper--Pearson \cite{clopper1934}
upper bound $U_-(r)$ on $q(r)$ and lower bound $L_+(r)$ on $\sigma(r)$, each with
one-sided noncoverage probability at most $\delta/(2J)$. By the union bound all $2J$ bounds hold simultaneously with
probability at least $1-\delta$, even though they share one validation sample.
Because $\Pi$ increases in $\sigma$ and decreases in $q$, for known $\Omega$
\begin{equation}
\label{eq:certprecision}
\underline\Pi(r)=\frac{L_+(r)}{L_+(r)+\Omega\,U_-(r)}
\end{equation}
is a simultaneous lower bound on precision (set to zero if the denominator
vanishes). If $\Omega$ is itself estimated, substitute an upper confidence bound
and spend part of $\delta$ on it; class-stratified samples do not estimate
prevalence. This is a simple instance of the learn-then-test principle
\cite{angelopoulos2025}: family-wise error control over a finite set of settings
permits a data-driven choice among those certified.

Choosing the smallest $r$ with $\underline\Pi(r)\ge\tau$ keeps the guarantee
despite the search over the shortlist, and among certified gates it keeps the
most good candidates. It may still be stricter than the population-optimal gate.
The data burden is severe. With zero observed false clearances, the exact upper
bound with noncoverage probability $\epsilon$ is $1-\epsilon^{1/N_-}$, and
certification requires
$1-\epsilon^{1/N_-}\le L_+(1-\tau)/(\tau\Omega)$. Take $\delta=0.05$ and
$J=20$, so $\epsilon=\delta/(2J)=0.00125$, with $\tau=0.99$ and $\Omega=100$.
If the sensitivity lower bound is $L_+=0.30$, approximately 221,000 unacceptable
validation candidates are needed; if $L_+=0.097$, approximately 682,000 are
needed. These calculations assume that the stated sensitivity lower bound has
already been established from the acceptable validation sample. Pooling across releases helps only if class-conditional
clearance probabilities stay stable. Two further cautions apply. Re-checking a fixed-sample
bound as data accumulate does not give an anytime-valid guarantee. And a bound on
population precision does not cover a winner chosen afterwards by another rule.

\section{Limitations}
\label{sec:limits}

Every quantitative statement in Sections~\ref{sec:findings}
and~\ref{sec:case} is conditional on the per-test error rates and the dependence
structure; Section~\ref{sec:certify} is the route to guarantees. We assume common
error rates across tests. Test-specific $p_j$ and $\eta_j$ replace the binomial
tail in \eqref{eq:clear} by a Poisson--binomial tail, which is easy to compute
but not explored here. A single factor imposes equicorrelation; real suites
can have several partly shared failure channels, such as tests that probe the
same capability or reuse the same data source, and a multi-factor version
requires multi-dimensional integration. Finally, the survival probability
$1-(1-\sigma)^{n_+}$ assumes independent candidates, while $q$, $\sigma$, $\Pi$
and the expected false-release count do not.

\section{Conclusion}

A release gate should be judged on two numbers: how trustworthy its clearances
are, and how many good models it throws away. Maximising the first alone always
leads to pass-all, because in the common-correlation model every extra required
pass raises precision. Among $r$-of-$k$ gates that meet the target, the most
lenient keeps the most good models; when no gate meets it, the rule says so
instead of defaulting to pass-all. Under pass-all gating, any precision target
below one is attainable within the model, but sensitivity tends to zero as the
suite grows.

Correlation between tests decides how expensive that target is. In our
configuration, raising the latent correlation from 0 to 0.5 turns an 8-test
pass-all problem into a 5,182-test one and leaves the gate keeping fewer than one good
model in ten. In an illustrative sweep, misjudging that correlation by a factor of
two costs more than moving the gate by one test. Teams should therefore measure
dependence within classes, not just per-test accuracy.

Model calculations show which designs are plausible. Labelled validation data,
collected under the stated sampling assumptions, can certify a population
precision target, and the simultaneous binomial bound lets a team pick a gate
from a shortlist without losing that guarantee. When no gate certifies, the
evidence does not establish that any shortlisted gate meets the target; the
cause may be too little data, a gate that truly falls short, or conservative
bounds. The data may not reveal which. A release decision is easier to defend
when it states what the evidence does establish and what it leaves open.

\end{document}